\documentclass{llncs}
\def\etal{\mbox{et al.}}
\usepackage{makeidx}  
\usepackage{graphicx}
\usepackage{amsmath,bm}
\usepackage{rotating}

\newcommand{\algacronym}{HMONN}
\newcommand{\algname}{Hierarchical Multi-Output Nearest Neighbor}

\begin{document}
\mainmatter              
\title{A Hierarchical Multi-Output Nearest Neighbor Model for Multi-Output Dependence Learning}
\titlerunning{A Hierarchical Multi-Output Nearest Neighbor Model}
\author{Richard G. Morris, Tony Martinez, Michael R. Smith}
\authorrunning{R.G. Morris, T. Martinez, and M.R. Smith} 
%
\tocauthor{Richard G. Morris, Tony Martinez, and Michael R. Smith}
\institute{ Brigham Young University, Provo, UT, 84602, USA,\\
\email{rmorris@axon.cs.byu.edu}, \email{martinez@cs.byu.edu}, \email{msmith@axon.cs.byu.edu}}

\maketitle

\begin{abstract}
	Multi-Output Dependence (MOD) learning is a generalization of standard classification problems that allows for multiple outputs that are dependent on each other.  A primary issue that arises in the context of MOD learning is that for any given input pattern there can be multiple correct output patterns.  This changes the learning task from function approximation to relation approximation.  Previous algorithms do not consider this problem, and thus cannot be readily applied to MOD problems.  To perform MOD learning, we introduce the \algname\ model (\algacronym) that employs a basic learning model for each output and a modified nearest neighbor approach to refine the initial results.  
\end{abstract}


\section{Introduction}
In this paper, we introduce Multi-Output Dependence (MOD) learning as an algorithmic family that models dependencies between multiple outputs.
Traditional supervised learning seeks to map an input vector $\bm{x}$ to an output vector $\bm{y} \in C$ where $C$ is set of possible outputs.
Further, multi-label classification specifically examines the case where multiple target labels must be assigned to each instance \cite{heath2010multiple} while structured prediction predict multiple target labels where $\bm{y}$ is structured.
MOD addresses problems where the outputs are dependent on each other and where there are multiple correct output vectors $\bm{y}$ for a given $\bm{x}$.
Any one output may be considered correct or incorrect only when considered in the context of other outputs.

An example MOD problem is the following.
Assume we want to propose an action for a company to take that generates sales and/or retains a customer.
For example, a particular customer may be contemplating switching to a competitor.
What should the company do to retain this customer?
There could be multiple correct actions.
A sales person could write the customer and offer incentives for staying, or the CEO could call the customer to express how important he is to them.
Of course, each action incurs a certain cost.
Having the CEO call a customer is more expensive than having a help-desk employee write the customer an e-mail, but both are viable solutions.
However, if that customer happens to be the largest source of revenue for that company, then sending a generic e-mail may not be the best course of action to take.
It may be the case where calling would only be correct \emph{if} the CEO made the phone call but writing an e-mail would be correct \emph{if} it came from a different person.



MOD problems can be seen as those problems where the outputs are important in addition to the inputs when making a decision.  
MOD learning requires approximating a relation, as opposed to the more traditional function approximation.  
We define a training data set $T$ to be a set of input vectors $\bm{x}$ each labeled with an appropriate output $\bm{y}$.  An input vector $\bm{x}$ can be associated with multiple output vectors $\bm{y}$ where $|\bm{y}| \geq 1$.  In this case, there are multiple correct outputs for $\bm{x}$.  Some outputs may be more desirable than others, but there are still multiple outputs that would be acceptable given the input $\bm{x}$.
This changes the task from finding a mapping function $\bm{x} \mapsto \bm{y}$ to finding a relation from $\bm{x}$ to $\bm{y}$.  We consider the relation where there are multiple outputs (where $|\bm{y}| > 1$) and there is a dependency between the outputs.  This gives rise to interesting questions about which correct solutions to choose when there are multiple correct solutions available.  



Many current algorithms fail to directly support multiple outputs.  Current approaches either induce one model per output or create a single model that gives multiple outputs without explicitly modeling the dependencies.  Different models support multiple independent outputs with a varying degree of success, without further modification.  Decision Tree learning algorithms must either induce multiple trees in order to produce multiple outputs or must induce a single tree that blows up exponentially, but neither of these approaches can model dependence between the output variables.  $K$-Nearest Neighbor algorithms can support multiple outputs with little change to the basic algorithm.  Multi-Layer Perceptron (MLP) models can give multiple outputs with a single model or multiple model approach.  However, none of these algorithms explicitly model dependent outputs.
Auto-associative models, such as Hopfield networks \cite{hopfield1985neural}, come close to this capability, but they are unable to handle arbitrary input and output mappings in contrast to the hetero-associative model that we present.

We introduce the \algname\ model\linebreak(\algacronym) in order to solve the MOD problem.  This hierarchical model has two layers.  The first layer is a na\"{i}ve approach with one learning model per output.  The models that comprise the first layer can be any traditional machine learning model.  The second layer is a modified nearest neighbor model that refines the predictions made on the first layer.  \algacronym~is shown graphically in Figure~\ref{fig:example}.  Though \algacronym\ focuses on tasks with nominal features, it also gives improvement for some tasks with real-valued features by implicitly modeling a similarity function for the feature space.


\begin{figure}
\centering
\includegraphics[width=0.45\textwidth]{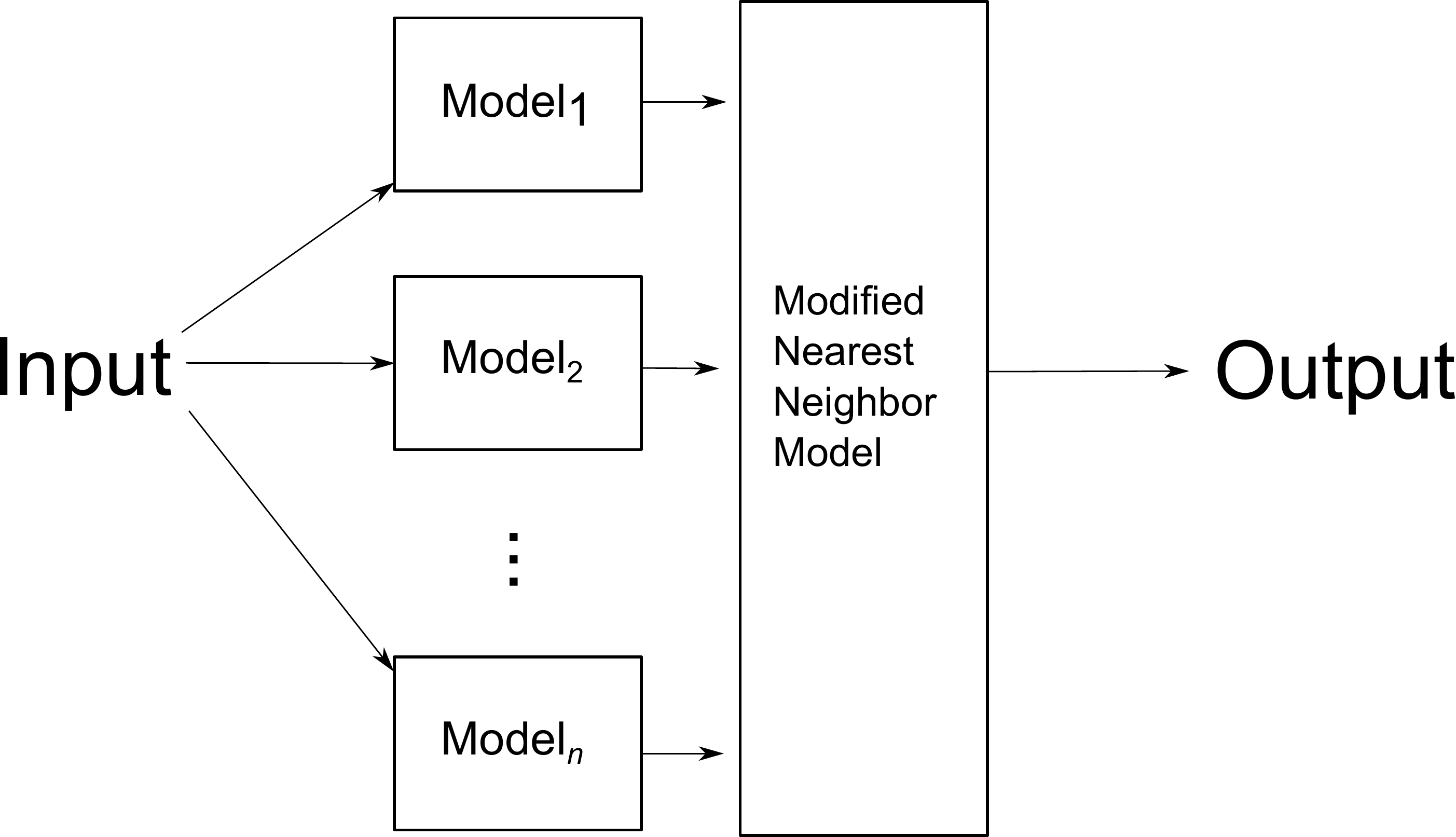}
\caption{The \algacronym\ model used to solve MOD problems.  The models used as input to the modified nearest neighbor model can use any algorithm to produce an initial prediction.  This models the dependence between the outputs in terms of the local context from the nearest neighbor algorithm.}
\label{fig:example}
\end{figure}


\section{Related Work}
\label{sec:relatedwork}

Other work has examined classification with multiple labels, although the labels are generally not considered to be dependent on each other and multiple correct labels are not considered.
Multi-label classification considers problems with multiple outputs, but no dependency between the outputs is modeled.
Tsoumakas \etal~\cite{tsoumakas2010mining} give an overview of multi-label classification.  They define two main approaches for multi-label classification.  The first approach is problem transformation, where the given data is transformed into a single problem that already has a well defined solution.  
The second approach is algorithm adaptation, where current algorithms are modified to solve the multi-label classification problem.  
Recent work has looked at correlations between labels in multi-label classification to improve accuracy \cite{Godbole04discriminativemethods}.
Read \cite{Read:2009:CCM:1617459.1617477,DBLP:journals/ml/ReadPHF11} introduced chain classifiers for supporting these correlations which could be viable for supporting MOD problems.

Many problems have a structure that is missed by standard classification algorithms \cite{taskar2003max}.  Structured Prediction (SP) seeks to solve this problem by modeling the structure of the outputs.  This structure could be a sequence, a tree, a graph, or an image.  This allows for multi-output as well as output dependencies.  However, these dependencies are almost always limited to Markovian dependencies --- related by time or space.  Theoretically, SP algorithms are capable of modeling any problem with structure, and MOD problems would be an example of this kind of problem.  The main difference between MOD and SP is that MOD problems are assumed to have some inputs with multiple correct outputs, whereas with current SP algorithms there is a single correct output assumed for each input.
Bakir \etal~\cite{bakir2007predicting} give an overview of the state of the art in SP.  

While MOD learning is relation approximation, this should not be confused with relational learning.  Statistical Relational Learning \cite{neville2003statistical,getoor2011learning} and Multi-Relational Learning \cite{dvzeroski2003multi} both handle relational data, not relation approximation.  These relational learning models learn a function from relational data and handle specially formatted and structured data.

\section{\algacronym}
\label{sec:methods}
We present the \textit{\algname}\ model (\algacronym) to solve the MOD problem. 
We define an output prediction $\hat{\bm{y}}$ to be correct for a given input vector $\bm{x}$ if there is some training instance in the training data $T$ that has $\bm{x}$ labeled with output $\bm{y}$ and $\hat{\bm{y}} = \bm{y}$ (or if $\hat{\bm{y}} = \bm{y}$ for the current test instance).  The traditional definition of a correct prediction only takes into account the labeling on the instance currently being tested.  This definition allows the model to use information contained within the training data to determine which output predictions should be counted as correct.

Traditional models are not able to model the dependencies between outputs.  This is, in part, due to the fact that traditional models are function approximators, and MOD problems are relational.  
As an example of this, consider a training set that contains two training patterns with the same $\bm{x}$ and different $\bm{y}$.
%
%
A traditional MLP will oscillate between the multiple possible outputs, and may not give any of the possible correct output vectors
An MLP will adjust weights towards outputting $\{1,0\}$ whenever the first instance is encountered, and whenever the second instance is encountered it will adjust the weights towards outputting $\{0,1\}$.  The network will consequently adjust the weights towards the output $\{0.5,0.5\}$ (without ever stabilizing), rather than towards either of the correct outputs.
A graphical example of the problem faced by an algorithm trying to learn a problem with multiple correct outputs is shown in Figure~\ref{fig:forked}.  The solid curve represents the relation in the training data and the dotted curve represents the function that could be learned, for example by an MLP.  An appropriate algorithm should output both branches of the relation, following the relation exactly, choose one of the branches arbitrarily, or choose one based on some criteria.  It should not, however, output something completely different.

\begin{figure}[t]
\centering
\begin{tabular}{cc}
\includegraphics[width=0.4\textwidth]{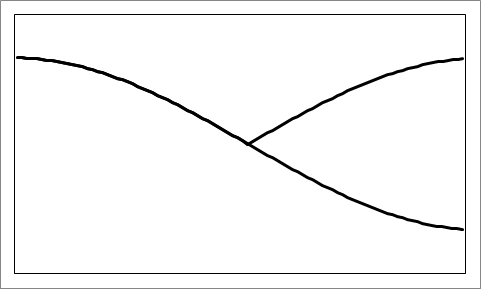} & \includegraphics[width=0.4\textwidth]{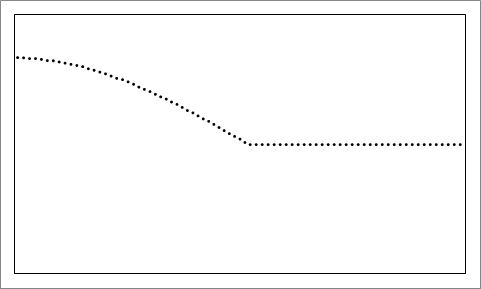}\\
Actual Relation & Possible Learned Function
\end{tabular}

\caption{A graphical example of a relation with multiple correct outputs.  The solid curve represents the relation itself.  The dotted curve represents the function learned by an MLP model.  The dotted curve follows the solid curve exactly until the relation branches, at which point the dotted curve is in the center of the two branches.}
\label{fig:forked}
\end{figure}

\algacronym~favors one output vector over the others.  Even though we could give a distribution of potential outputs from the neighborhood of the initial prediction, this version gives one of the possible correct output vectors for the given input vector $\bm{x}$.  This output vector is the most common among the given neighborhood, and thus varies with neighborhood size and makeup.
%
%
\algacronym\ is a first step towards solving the MOD problem.  
\algacronym\ starts with an initial prediction and then uses the extra information provided by that initial prediction to give the output.  The initial prediction is obtained using any machine learning method.  
Here, we choose to train an MLP classifier for each output.  The outputs from each MLP classifier are combined into an initial prediction.  We present a modified \emph{K}-Nearest Neighbor (KNN) algorithm to give the final prediction.  
%
%
\algacronym\ uses a different distance function where the initial prediction is used as part of the features for the distance function:
\begin{equation}
Dist(\{\bm{x}_1,\bm{y}_1\}, \{\bm{x}_2, \bm{y}_2\}) = \sqrt{\theta\sum_{i=1}^{N}{(x_{1,i} - x_{2,i})^2} + (1 - \theta)\sum_{i=1}^{M}{(y_{1,i} - y_{2,i})^2}}
\end{equation}
where $N$ is the number of features in the input space, $M$ is the number of outputs, and $\theta$ is a weight on the range $[0,1]$. 
The value for $\theta$ emphasizes either the input space or the output space as more important.  This modification of KNN captures the dependency between output variables by incorporating them into the input feature space.  
\algacronym\ takes the initial prediction from the MLP classifiers, uses this initial prediction as part of the features in a KNN algorithm, and chooses the majority output vector from the neighborhood as the final prediction.

The relationship between the dependence between outputs and multiple correct output vectors for a given input vector is shown in Theorem~\ref{thm:multiout}.
Theorem~\ref{thm:multiout} claims that we can observe the dependence between two output variables directly in the training data.  There is also a case for loose dependence that relies on different $\bm{x}$ vectors being only similar, but this work considers exact equality.

\begin{theorem}
\label{thm:multiout}
Given random variables $\bm{x}$, $y_i$, and $y_j$, where $\bm{x}$ is an input vector of nominal features and $y_i$ and $y_j$ are scalars from the output vector $\bm{y}$, if the two output variables, $y_i$ and $y_j$, are conditionally dependent on each other given the input $\bm{x}$ and the training data $T$, then there is some input vector, $\bm{x}$, associated with multiple output vectors, $\bm{y}$, in $T$.
\end{theorem}
\begin{proof}
Assume that outputs $y_1$ and $y_2$ are conditionally dependent given the input variable $X$ and the training data $T$.  By the definition of statistical dependence this implies that, for some input vector $\bm{x}$, $P(y_1 \mid y_2, \bm{x}, T) \neq P(y_1 \mid \bm{x}, T)$.  Assume that the output vector $\bm{y} = [y_1 , y_2]^T$ is the only possible correct output for $\bm{x}$.  Then it is the case that $P(y_1 \mid y_2, \bm{x}, T)~=~P(y_1 \mid \bm{x}, T) = 1$.  This contradicts the definition of statistical dependence.  Thus, there must be multiple possible output vectors for the input vector $\bm{x}$. \qed
\end{proof}

\section{Experimental Results}
\label{sec:experiments}
The accuracy of MOD classifiers was evaluated on three different types of data: synthetic data, UCI repository data, and real-world data.  This accuracy was compared to a baseline model that consists of a single classifier trained separately for each output, which we call the na\"ive model, where each separate prediction is combined into a single output vector.  Accuracy is defined as follows.

\begin{equation}
MOD\_accuracy = \frac{\sum_i^D I(\{\bm{x}_i,\bm{z}_i\} \in T \cup \{D_i\})}{|D|}
\label{eqn:modaccuracy}
\end{equation}
where $D$ is the test set, $\bm{z}_i$ is the predicted output vector for instance $\bm{x}_i$, $T$ is the training set, and $I(x)$ is the Kronecker delta function returning 1 if the expression $x$ is true and 0 otherwise.  This accuracy metric counts a prediction as correct if an input vector $\bm{x}$ in the data set is labeled with the predicted output vector $\bm{z}$.  This considers all correct output vectors as equally good.

Standard machine learning tasks with only nominal input features are common, and we assume that the same will hold for MOD data sets.  \algacronym\ shows clear improvement on these data sets.  Many tasks also have real-valued features.  While it is more difficult to find a duplicate $\bm{x}$ in these data sets, real-valued features will often have some level of discretization done to them, through either binning or rounding that increases the likelihood of finding duplicate $\bm{x}$ vectors in the data set.  This alters the amount of dependence between the output variables. 
Thus, in many current data sets, real-valued features do not necessarily take on a large range of values.  This allows the given definition of accuracy to work in many cases with real-valued features.
To better handle real-valued features, the definition of accuracy could be extended to allow for \emph{similar} values, as opposed to requiring values to be exactly equal.  
We are currently working on extending MOD accuracy metrics to better support real-valued features.


Despite the issue of the frequency of exact $\bm{x}$ vectors for real-valued features, \algacronym\ improves the accuracy in some of the experiments on synthetic and UCI data that have real-valued  inputs.
This is due to the fact that the nearest neighbor portion of the algorithm creates an implicit similarity function for the feature space.  The similarity function behaves differently based on the neighborhood size.  This gives a distance-based voting for which outputs are correct for any given portion of the feature space.  This causes the majority class for any given neighborhood in the feature space to always be the correct value.  
Selecting outputs in this fashion avoids some of the difficulty with real-valued features, even though it does not solve the problem completely.  We are currently exploring ways to fully resolve this problem as future work. 

Some initial experimentation was used to determine values for $k$ and $\theta$.  We tested values of $k$ from 1 to 11 and values of $\theta$ from $\{0,0.25,0.5,0.75,1\}$. We found that there was little difference between values of $k$ and $\theta$ except for $\theta = 0$, which performed slightly worse. 
In the following experiments, we use representative values $k = 7$, allowing for a reasonably sized neighborhood, and $\theta = 0.5$, to give an equal balance between the input and output features. 
%
Experiments are run using 10-fold cross validation.  The na{\"i}ve neural network layer had a standard MLP with a single hidden layer of $2n$ nodes for each output with $n$ being the number of attributes, including the outputs, in the corresponding data set.  All experiments are run with a learning rate of 0.1 and stop after 10 epochs without any improvement on a held-out validation set.  Statistical significance is determined using the Wilcoxon signed rank test with significance at $p < 0.05$.
%

\subsection{Synthetic Data}

Two different types of synthetic data were created.  One used real-valued features in order to determine whether \algacronym\ implicitly models a similarity function for the feature space, as hypothesized.  The other used only nominal features.

Real-valued synthetic data was created using the following process.  Given $o$ output variables, a data set is generated by selecting $c$ points in the input space as centroids.  These points are each randomly assigned a number of probability vectors.  A probability vector contains a probability distribution over possible output vectors.  To generate an instance, a centroid is selected at random, the input values for that instance are generated by randomly perturbing the centroid according to a Gaussian distribution.  An output vector is chosen by randomly selecting an output vector according to the probability distribution of a randomly chosen probability vector for that centroid.  
This generation process attempts to model the fact that, for MOD problems, a portion of the input space can belong to more than one output vector.
Nominal synthetic data was created using the process outlined above with one difference.  To generate a centroid, a center point for each feature was chosen from $\{0,1,2,3\}$.  New inputs were generated by adding a randomly selected value from $\{-1,0,+1\}$ to the center point for that feature.  Values above $3$ were set to $3$ and values below $0$ were set to $0$.
%
%
The parameters were set to $o \in \{2,3,4\}$ (with 4 possible values for each output) and $c \in \{2,4,6,8\}$.  The number of inputs was set to 3 times the number of outputs.  The number of probability vectors was the same as the number of centroids, 1.5 times the number of centroids, or 2 times the number of centroids.  5000 instances were generated for each data set.  This results in 12 data sets for each of 2 outputs, 3 outputs, and 4 outputs, giving a total of 36 data sets used.


\begin{table}[t]
\begin{center}
\caption{Results comparing \algacronym\ to the na{\"i}ve model for real-valued features.  Bold values indicate that the values are statistically significant.  The $p$-value for the total is $p <.0001$.}
\label{table:syntheticcomparison}
    \begin{tabular}{|c|c|c|c|c||c|c|c|c|}
 \multicolumn{1}{c}{} & \multicolumn{4}{c}{Real-Valued Features} & \multicolumn{4}{c}{Nominal Features} \\
        \hline
 Model& 2-Output  & 3-Output  & 4-Output  & Total &  2-Output  & 3-Output  & 4-Output  & Total\\ \hline
 \algacronym\  & \textbf{0.760} & \textbf{0.877} & \textbf{0.905} & \textbf{0.847} & \textbf{0.703} & \textbf{0.864} & \textbf{0.893} & \textbf{0.820}\\
 Na{\"i}ve  & 0.718 & 0.770 & 0.762 & 0.750 & 0.655 & 0.758 & 0.713 & 0.709\\
        \hline
    \end{tabular}
\end{center}
\end{table}

The results of comparing \algacronym\ to the na{\"i}ve model for real-valued and nominal features are given in Table \ref{table:syntheticcomparison}.
\algacronym\ outperformed the na{\"i}ve model for the real-valued synthetic data, and the improvement was always statistically significant.  This is likely due to the fact that \algacronym\ exploits the information contained in the local neighborhood in order to produce outputs.  \algacronym\ will have more information available with more outputs.  This will make the neighborhood more specific, thus giving the algorithm a higher chance of finding a correct output.
%
%
%
\algacronym\ outperformed the na{\"i}ve model for the nominal synthetic data as well, and the improvement was always statistically significant.  

\subsection{UCI Data}


The UCI repository \cite{FrankAsuncion:2010} does not contain any data sets that are MOD decision problems.  Therefore, MOD data sets were created from the original UCI Data sets by allowing each nominal feature to act as an output class for a derivative data set.  If, for example, the number of outputs was set to two, each data set would create $n$ derived data sets where $n$ is the number of nominal features for the chosen data set.  Each of these derived data sets consists of a nominal feature combined with the original output class acting as the output classes, with all of the other features acting as inputs.  Similarly, for three or four outputs the original output class is combined with two or three (respectively) nominal features to act as the outputs.  The number of data sets scales linearly in the number of inputs with two outputs, quadratically with three outputs, and cubicly with four outputs.  This is a contrived solution, but we assume that there is some dependency between input variables and the output variable --- especially for data sets from the UCI repository.  
Twenty UCI data sets were used for the experiments.  These data sets were chosen arbitrarily from those that had more than five nominal input features.  Nominal input features were necessary in order to create the derivative data sets.  
Information for each data set is provided in Table \ref{table:ucidata}.
%
%
Missing values were replaced by the mean/mode.

\begin{table}[t]
\centering
\begin{center}
\caption{Results for the UCI experiments.  The H columns signify the accuracy for \algacronym\, and the NI columns signify the accuracy for the na{\"i}ve independence model.  Bold indicates that the model had significantly greater accuracy.  The \emph{\# Significant} line indicates how many of the entries in each column were statistically significant. The \emph{info} provides information about the UCI data sets including the number of features (Feat) and nominal features (N), and the number of derived 2, 3, and 4-output data sets.}
\label{table:ucidata}
    {\begin{tabular}{| l | c | c || c | c || c | c || c | c |ccccc|}
        \hline
~ & \multicolumn{2}{|c||}{2-Output} &  \multicolumn{2}{|c||}{3-Output} & \multicolumn{2}{|c||}{4-Output} & \multicolumn{2}{|c|}{Total} & \multicolumn{5}{|c|}{Info} \\ \hline
Data set & H & NI & H & NI & H & NI & H & NI & Feat & N & 2-O & 3-O & 4-O\\ \hline
adult & 0.255 & 0.268 & 0.136 & 0.139 & 0.075 & \textbf{0.083} & 0.109 & \textbf{0.116} & 14 & 8 & 8 & 28 & 56 \\
anneal & \textbf{0.756} & 0.511 & \textbf{0.700} & 0.364 & \textbf{0.659} & 0.254 & \textbf{0.664} & 0.265& 38 & 32 & 32 & 496 & 4960 \\
autos & \textbf{0.265} & 0.192 & \textbf{0.208} & 0.118 & \textbf{0.132} & 0.068 & \textbf{0.159} & 0.088 & 25 & 10 & 10 & 45 & 120 \\
car & \textbf{0.233} & 0.229 & \textbf{0.066} & 0.063 & \textbf{0.018} & 0.017 & 0.067 & 0.065 & 6 & 6 & 6 & 15 & 20 \\
chess & \textbf{0.100} & 0.090 & \textbf{0.020} & 0.017 & \textbf{0.004} & 0.003 & \textbf{0.024} & 0.021 & 6 & 6 & 6 & 15 & 20 \\
cmc & \textbf{0.358} & 0.281 & \textbf{0.243} & 0.203 & \textbf{0.173} & 0.132 & \textbf{0.217} & 0.172 & 9 & 7 & 7 & 21 & 35 \\
colic & \textbf{0.330} & 0.255 & \textbf{0.174} & 0.106 & \textbf{0.106} & 0.048 & \textbf{0.124} & 0.064 & 22 & 15 & 15 & 105 & 455 \\
credit-a & 0.424 & 0.434 & 0.278 & \textbf{0.297} & 0.178 & \textbf{0.202} & 0.223 & \textbf{0.245} & 15 & 9 & 9 & 36 & 84 \\
crx & \textbf{0.441} & 0.438 & 0.265 & \textbf{0.290} & 0.161 & \textbf{0.191} & 0.210 & \textbf{0.236} & 15 & 9 & 9 & 36 & 84\\
heart-h & \textbf{0.509} & 0.279 & \textbf{0.401} & 0.152 & \textbf{0.300} & 0.067 & \textbf{0.357} & 0.119 & 13 & 7 & 7 & 21 & 35 \\
hepatitis & 0.591 & 0.605 & 0.443 & 0.431 & \textbf{0.381} & 0.341 & \textbf{0.401} & 0.369 & 19 & 13 & 13 & 78 & 286 \\
mush & \textbf{0.795} & 0.775 & \textbf{0.630} & 0.590 & \textbf{0.497} & 0.451 & \textbf{0.518} & 0.473 & 22 & 22 & 22 & 231 & 1540 \\
nursery & \textbf{0.282} & 0.278 & \textbf{0.091} & 0.089 & \textbf{0.028} & 0.027 & \textbf{0.069} & 0.067 & 8 & 8 & 8 & 28 & 56 \\
poker & \textbf{0.095} & 0.091 & \textbf{0.016} & 0.015 & \textbf{0.003} & 0.002 & \textbf{0.011} & 0.011 & 10 & 10 & 10 & 45 & 120 \\
post-op & \textbf{0.360} & 0.346 & \textbf{0.248} & 0.238 & 0.128 & 0.136 & 0.194 & 0.194 & 8 & 7 & 7 & 21 & 35 \\
SPECT & \textbf{0.661} & 0.651 & \textbf{0.513} & 0.494 & 0.397 & 0.372 & 0.415 & 0.391 & 22 & 22 & 22 & 231 & 1540 \\
TA & 0.217 & 0.183 & \textbf{0.139} & 0.061 & 0.088 & 0.017 & \textbf{0.146} & 0.083 & 5 & 4 & 4 & 6 & 4 \\
tic-tac & \textbf{0.453} & 0.425 & 0.190 & 0.187 & 0.065 & \textbf{0.074} & 0.127 & \textbf{0.130} & 9 & 9 & 9 & 36 & 84 \\
vote & \textbf{0.729} & 0.708 & \textbf{0.564} & 0.540 & \textbf{0.442} & 0.417 & \textbf{0.470} & 0.445 & 16 & 16 & 16 & 120 & 560 \\
zoo & \textbf{0.816} & 0.568 & \textbf{0.726} & 0.492 & \textbf{0.654} & 0.400 & \textbf{0.670} & 0.420 & 16 & 16 & 16 & 120 & 560 \\ \hline \hline
\# Sig & 16 & 0 & 15 & 2 & 13 & 4 & 13 & 4 &&&&&\\
        \hline
    \end{tabular}}
\end{center}
\end{table}


Table~\ref{table:ucidata} shows the results for the UCI data set experiments.  The table contains values for both \algacronym\ and the na\"ive model compared by number of outputs.
Each number is obtained by averaging the results across all the derived data sets from the original UCI data set for the given algorithm.  Statistically significant results are highlighted.  \algacronym\ outperformed the na\"ive model 79\% of the time (with 68\% of the time being statistically significant, see the Total columns).  In some cases there was not a significant difference.  In four cases, the na\"ive model outperformed \algacronym.  
\algacronym\ outperforms the na\"ive model in the majority of cases.  Occasionally, the na\"ive model performs better, but never with the same magnitude.  This further demonstrates the potential of \algacronym\ as a model to solve MOD decision problems.  This also validates the assumption that there is some dependence between the output variable and the input variables in the UCI data sets.


\subsection{Business Application Data}

The motivation for defining MOD problems stems from
a local business, InsideSales.com, that provided data for a real world MOD task.  Due to the proprietary nature of this business data, we are only permitted to reproduce a de-identified version of this data.  This data includes a two output data set and a three output data set.  The data sets have fourteen nominal features and eight real-valued features.  The two output data set has 32544 instances, and the three output data set has 32774 instances.  


The task is to determine the timing and method to contact business leads.  Business practices would imply that these variables are dependent (given the input $\bm{x}$), the time you contact a lead depends on the method used, and the method used depends on the timing.  
The results are shown in Table~\ref{table:realdata}.  \algacronym\ outperformed the na{\"i}ve model in both cases.  This shows that the improvement of \algacronym\ seen in the UCI and synthetic data can also be seen in real-world MOD problems.
The synthetic data and the real-world business data are definitely MOD problems. However the synthetic data is not necessarily representative of real data, and there is little real data.  The UCI data is used to supplement the other data sources, although it can only be assumed to represent MOD data.


\begin{table}[t]
\begin{center}
\caption{The table showing the results of the real-world business data experiments.}
\label{table:realdata}
    {\begin{tabular}{|l|c|c||l|c|c|}
        \hline
        ~              & \algacronym\     & Na{\"i}ve & ~              & \algacronym\     & Na{\"i}ve\\ \hline
        2-Output             & 0.508 & 0.465    & 3-Output              & 0.346  & 0.307 \\
         
        \hline
    \end{tabular}}
\end{center}
\end{table}

\section{Conclusions}
\label{sec:conclusions}
We provided a definition for MOD problems, as a well as a method to solve such problems.  We have defined the \algname\ model, with a na\"ive independence model as a first layer and a modified nearest neighbor model as the second layer.  This model is based on the assumption that local context is a key element to solving MOD problems.  \algacronym\ consistently outperforms the baseline model, typically with statistical significance.  This holds true for synthetic data, UCI repository data, and for one real-world business task.  

Future work will develop solutions using other types of models (such as relaxation networks), an improved method for calculating accuracy on MOD problems, improved methods for validating new MOD algorithms, and new methods for identifying and collecting MOD data.  With MOD problems, it is difficult to know how much dependency any given problem may have.  Many of the data sets that we used for validation could only be assumed to have some level of dependency.  A method to identify the degree of output dependency in a given data set is another piece of future work.

\bibliographystyle{splncs03}

\bibliography{proposal}

\newcommand{\noopsort}[1]{} \newcommand{\printfirst}[2]{#1}
  \newcommand{\singleletter}[1]{#1} \newcommand{\switchargs}[2]{#2#1}
\begin{thebibliography}{10}
\providecommand{\url}[1]{\texttt{#1}}
\providecommand{\urlprefix}{URL }

\bibitem{bakir2007predicting}
Bak{\i}r, G., Hofmann, T., Sch{\"o}lkopf, B.: Predicting structured data. The
  MIT Press (2007)

\bibitem{dvzeroski2003multi}
D{\v{z}}eroski, S.: Multi-relational data mining: an introduction. ACM SIGKDD
  Explorations Newsletter  5(1),  1--16 (2003)

\bibitem{FrankAsuncion:2010}
Frank, A., Asuncion, A.: {UCI} machine learning repository (2010),
  \url{http://archive.ics.uci.edu/ml}

\bibitem{getoor2011learning}
Getoor, L., Mihalkova, L.: Learning statistical models from relational data.
  In: Proceedings of the 2011 international conference on Management of data.
  pp. 1195--1198. ACM (2011)

\bibitem{Godbole04discriminativemethods}
Godbole, S., Sarawagi, S.: Discriminative methods for multi-labeled
  classification. In: Proceedings of the 8th Pacific-Asia Conference on
  Knowledge Discovery and Data Mining. pp. 22--30. Springer (2004)

\bibitem{heath2010multiple}
Heath, D., Zitzelberger, A., Giraud-Carrier, C.: {A Multiple Domain Comparison
  of Multi-label Classification Methods}. Working Notes of the 2nd
  International Workshop on Learning from Multi-Label Data p.~21 (2010)

\bibitem{hopfield1985neural}
Hopfield, J., Tank, D.: {Neural computation of decisions in optimization
  problems}. Biological cybernetics  52(3),  141--152 (1985)

\bibitem{neville2003statistical}
Neville, J., Rattigan, M., Jensen, D.: Statistical relational learning: Four
  claims and a survey. In: Workshop SRL, Int. Joint. Conf. on AI (2003)

\bibitem{Read:2009:CCM:1617459.1617477}
Read, J., Pfahringer, B., Holmes, G., Frank, E.: Classifier chains for
  multi-label classification. In: Proceedings of the European Conference on
  Machine Learning and Knowledge Discovery in Databases: Part II. pp. 254--269.
  Springer-Verlag (2009)

\bibitem{DBLP:journals/ml/ReadPHF11}
Read, J., Pfahringer, B., Holmes, G., Frank, E.: Classifier chains for
  multi-label classification. Machine Learning  85(3),  333--359 (2011)

\bibitem{taskar2003max}
Taskar, B., Guestrin, C., Koller, D.: Max-margin markov networks. In: Thrun,
  S., Saul, L., Sch\"{o}lkopf, B. (eds.) Advances in Neural Information
  Processing Systems 16. MIT Press (2004)

\bibitem{tsoumakas2010mining}
Tsoumakas, G., Katakis, I., Vlahavas, I.: {Mining multi-label data}. Data
  Mining and Knowledge Discovery Handbook pp. 667--685 (2010)

\end{thebibliography}

\end{document}